%% file: arxiv_advere-duel.tex
\documentclass[11pt]{article}
\usepackage{fullpage}

\title{\bfseries\papertitle}

\author{
Aadirupa Saha%
\thanks{Indian Institute of Science, Bangalore, India; {\tt aadirupa@iisc.ac.in}.}
\and 
Tomer Koren%
\thanks{Tel Aviv University and Google Tel Aviv; {\tt tkoren@tauex.tau.ac.il}.} 
\and Yishay Mansour%
\thanks{Tel Aviv University and Google Tel Aviv; {\tt mansour.yishay@gmail.com}.}
}

\usepackage{microtype}
\usepackage{amsmath,amssymb,amsthm}

\usepackage{color}
\usepackage{cancel}
\usepackage{graphicx}
\usepackage{booktabs} 

\usepackage{thmtools}
\usepackage{thm-restate}

\usepackage{enumitem}
\usepackage{algorithm}
\usepackage{algorithmic}
\usepackage[colorlinks,allcolors=blue]{hyperref}
\usepackage[capitalize]{cleveref}

\allowdisplaybreaks

\theoremstyle{plain}
\newtheorem{thm}{Theorem}

\theoremstyle{definition}
\newtheorem{defn}[thm]{Definition}

\theoremstyle{remark}

\newcommand{\R}{{\mathbb R}}

\renewcommand{\P}{{P}}

\newcommand{\E}{{\mathbf E}}

\newcommand{\1}{{\mathbf 1}}

\newcommand{\cA}{{\mathcal A}}

\newcommand{\cH}{{\mathcal H}}

\newcommand{\cI}{{\mathcal I}}

\newcommand{\cE}{{\mathcal E}}

\newcommand{\ts}{{\tilde s}}
\newcommand{\tK}{{\tilde K}}
\newcommand{\tO}{{\tilde O}}

\renewcommand{\b}{{\mathbf b}}

\newcommand{\bb}{{\bar b}}
\newcommand{\tb}{{\tilde b}}
\newcommand{\hb}{{\hat b}}

\newcommand{\q}{{q}}
\newcommand{\tq}{\tilde \q}

\newcommand{\sm}{\setminus}

\DeclareMathOperator*{\argmax}{arg\,max}

\def \papertitle {{Adversarial Dueling Bandits}}
\def \prob {{Adversarial Dueling Bandits}}
\def \obj {{Borda Score}}
\def \gap {{Fixed-Gap}}
\def \acr {{Adv-Borda}}

\newcommand{\tk}[1]{\textcolor{magenta}{\bfseries \{TK:~#1\}}}

\newcommand{\algdexp}{Dueling-EXP3}
\newcommand{\alggap}{Borda-Confidence-Bound}

\begin{document}

\maketitle

\input{abstract.tex}

\input{introduction.tex}

\input{problem.tex}


\input{alg_borda.tex}
\input{alg_borda_whp.tex}

\input{alg_fxdgap.tex}

\input{lower_bound.tex}

\input{conclusion.tex}

\subsection*{Acknowledgements}

This project has received funding from the European Research Council (ERC) under the European Union’s Horizon 2020 research and innovation program (grant agreement No. 882396), and from the Israel Science Foundation (grants 993/17 and 2549/19). AS thank Qualcomm Innovation Fellowship IND-417067, 201.

\bibliographystyle{plain}
\bibliography{adversarial-dueling}


\input{appendix.tex}

\end{document}

%% file: abstract.tex
\begin{abstract}
We introduce the problem of regret minimization in Adversarial Dueling Bandits. As in classic Dueling Bandits, the learner has to repeatedly choose a pair of items and observe only a relative binary `win-loss' feedback for this pair, but here this feedback is generated from an arbitrary preference matrix, possibly chosen adversarially. 
Our main result is an algorithm whose $T$-round regret compared to the \emph{Borda-winner} from a set of $K$ items is $\tilde{O}(K^{1/3}T^{2/3})$, as well as a matching $\Omega(K^{1/3}T^{2/3})$ lower bound. We also prove a similar high probability regret bound.
We further consider a simpler \emph{fixed-gap} adversarial setup, which bridges between two extreme preference feedback models for dueling bandits: stationary preferences and an arbitrary sequence of preferences. For the fixed-gap adversarial setup we give an $\smash{ \tilde{O}((K/\Delta^2)\log{T}) }$ regret algorithm, where $\Delta$ is the gap in Borda scores between the best item and all other items, and show a lower bound of $\Omega(K/\Delta^2)$ indicating that our dependence on the main problem parameters $K$ and $\Delta$ is tight (up to logarithmic factors).
\end{abstract}

%% file: introduction.tex
\section{Introduction}
\label{sec:intro}

\emph{Dueling Bandits} is an online decision making framework similar to the
well known (stochastic) multi-armed bandit (MAB)
problem~\cite{auer02,Slivkins19}, that has gained widespread attention in the
machine learning community over the past
decade~\cite{Yue+12,Zoghi+14RCS,Zoghi+15}.
%
%
In Dueling Bandits, a learner repeatedly selects a pair of items to be
compared to each other in a ``duel,'' and consequently observe a binary stochastic
preference feedback, which can be interpreted as the winning item in this duel.
The goal of the learner is to minimize the regret with respect to the best item
in hindsight, according to a certain score function.

%
%
Numerous real-world applications are naturally modelled as dueling bandit
problems, including movie recommendations,
tournament ranking, search engine optimization, retail management, etc. (see also \cite{Busa14survey,Yue+09}). 
Indeed,
in many of these scenarios, users with whom the algorithm interacts with find it more
natural to provide binary feedback by comparing two alternatives rather than
giving an absolute score for a single alternative.
Over the years, several algorithms have been proposed for addressing dueling
bandit problems~\cite{Ailon+14,Zoghi+14RUCB,Komiyama+15,Zoghi+14RCS} and there
has been some work on extending the pairwise preference to more general
subset-wise preferences~\cite{Sui+17,Brost+16,SG18,SGwin18,Ren+18}.

While almost all of the existing literature on dueling bandits focus on
stochastic \emph{stationary} preferences, in reality preferences might vary
significantly and unpredictably over time.
For example, in movie recommendation systems, user preferences may evolve
according to daily and hourly viewing trends; in web-search optimization,
relevance of various websites may vary rather unpredictably.
In other words, many of the real-world applications of dueling bandits actually
deviate from the stochastic feedback model, and would more faithfully be
modelled in a robust worse-case (adversarial) model that alleviates the strong
stochastic assumption and allows for an arbitrary sequence of preferences over time.
For similar reasons, the MAB problem, and more generally, online learning, are
frequently studied in a non-stochastic adversarial
setup~\cite{CsabaNotes18,BubeckNotes+12,PLG06,seldin+14,seldin+17,neu15,bubest12}.

Surprisingly, however, a non-stochastic version of dueling bandits has not been well studied (with the only exception being \cite{Adv_DB}, discussed below.) The first challenge in eschewing stationarity in dueling bandits lies in the performance benchmark compared to which regret is defined. Indeed, most works on stochastic dueling bandits rely on the existence of a \emph{Condorcet winner}: an item being preferred (and often by a gap) when compared with any other item.
In an adversarial environment, however, assuming a Condorcet winner makes  little sense as it would constrain the adversary to consistently prefer a certain item at all rounds, ultimately defeating the purpose of a non-stationary model in the first place. Another main challenge is the inherent disconnect between the feedback observed by the learner and her payoff at any given round; while this disparity already exists in stochastic models of dueling bandits, in an adversarial setup it becomes more tricky to attribute preferential information to the \emph{instantaneous} quality of items.

\paragraph{Our contributions.}

In this paper, we introduce and study an adversarial version of dueling bandits.
To mitigate the issues associated with Condorcet assumptions, and following recent literature on dueling bandits (e.g., \cite{Jamieson+15,Ramamohan+16,falahatgar_nips}), we focus on the so-called Borda criterion. 
The \emph{Borda score} of an item is the probability that it is preferred over another item chosen uniformly at random.
A \emph{Borda winner} (i.e., an item with the highest Borda score) always exists for any preference matrix, and more generally, this notion naturally extends to any arbitrary sequence of preference matrices.
However, the second challenge from above remains: the Borda score of an item is not directly related in nature to the preferential feedback observed for this item on rounds where it is chosen for a duel.


The main contributions of this paper can be summarized as follows:
\begin{itemize}[leftmargin=!]
\item 
We introduce and formalize an adversarial model for $K$-armed dueling bandits with standard binary ``win-loss'' preferential feedback (and where regret is measured with respect to Borda scores).
To the best of our knowledge, we are first to study such a setup.

\item
In the general adversarial model, where the sequence of preference matrices is allowed to be entirely arbitrary, we present an algorithm with expected regret bounded by
$\smash{ \tO(K^{1/3}T^{2/3}) }$.\footnote{Throughout, the notation $\smash{ \tilde O(\cdot) }$ hides logarithmic factors.} 
We further demonstrate how to modify our algorithm so as to guarantee a similar bound with high probability.
We also give a lower bound of $\smash{ \Omega(K^{1/3}T^{2/3}) }$, showing our algorithm is nearly optimal.

\item
We consider a more specialized fixed-gap adversarial model, that bridges between the two extreme preference feedback models for dueling bandits: the well-studied \emph{stationary} stochastic preferences, and \emph{fully adversarial} preferences.
Here, we assume that there is a fixed item whose average Borda score at any point in time exceeds that of any other item
by at least $\Delta > 0$, where $\Delta$ is a gap parameter unknown to the learner. (Other than constraining this fixed gap, the preference assignment may change adversarially.)
We present an algorithm that achieves regret $\smash{ \tilde{O}(K/\Delta^2) }$, and show that it is near-optimal by proving a regret lower bound of $\Omega(K/\Delta^2)$.

\end{itemize}

Our results thus reveal an inherent gap in the achievable regret between dueling bandits and standard multi-armed bandits:
in the adversarial model, the optimal regret in dueling bandits grows like $\smash{\Theta(T^{2/3})}$ whereas in standard bandits $\smash{ \Theta(\sqrt T) }$-type bounds are possible;
likewise, in the fixed-gap model the optimal regret for dueling bandits is $\smash{ \tilde{\Theta}({K}/{\Delta^2}) }$, versus the well-known $\tilde{\Theta}(K/\Delta)$ regret performance for standard fixed-gap (stochastic) bandits.

The reason for this substantial gap, as we explain in more detail in our discussion of lower bounds, is the following. For gaining information about the identity of the best item in terms of Borda scores, the learner might be forced to choose items the scores of which are already (or even initially) known to be suboptimal, and for which she would unavoidably suffer constant regret. Indeed, the Borda score of an item inherently depends on its relative performance compared to \emph{all other items}, and it may be that the identity of the Borda winner is determined solely by its comparison to poorly-performing items.




\paragraph{Related work.}

Dueling bandits were investigated extensively in the stochastic setting.
The most frequently used performance objective in this literature is the regret compared to the \emph{Condorcet Winner}~\cite{Yue+12,Zoghi+14RUCB,Zoghi+15MRUCB,Komiyama+15,BTM}. However, there are quite a few well-established shortcomings of this objective; most importantly, the Condorcet winner often fails to exist even for a fixed preference matrix. (See \cite{Jamieson+15} for more detailed discussion.)
In absence of Condorcet winners, there are other preference notions studied in the literature, most notably the~\emph{Borda Winner}~\cite{Busa14survey,Jamieson+15,Ramamohan+16,falahatgar_nips}, \emph{Copeland Winner}~\cite{Zoghi+15,komiyama+16,DTS},\footnote{It is worth noting that for the Copeland winner to be at all learnable, a gap assumption is required.} and \emph{Von-Neumann Winner}~\cite{CDB,Balsubramani+16}. In this work, we focus on the Borda Winner, which appears to be the most common alternative.

The only previous treatment of dueling bandits in an adversarial setting appears to be \cite{Adv_DB}, which considers utility-based preferences and thereby imposes a complete ordering of the items in each time step rather than a general preference matrix. Further, their feedback model includes not only the winning item but also a transfer function which is the difference in utilities between the compared items, thus being more similar to standard MAB and largely departs from the original motivation of dueling bandit. 
For the identity transfer function, they show in their adversarial utility-based dueling bandit model a tight regret bound of $\smash{ \tilde{\Theta}(\sqrt{KT}) }$. In contrast, we show for the adversarial dueling bandit model a tight regret bound of $\smash{ \tilde{\Theta}(K^{1/3}T^{2/3}) }$. This shows that when one does not have a direct access to a transfer function and is faced with arbitrary preferences, the regret scales substantially different, i.e., $\smash{ \tilde{\Theta}(T^{2/3}) }$ versus $\smash{ \tilde{\Theta}(T^{1/2}) }$.

The work \cite{Jamieson+15} shows an instance dependent $\tilde\Omega(K/\Delta^2)$ sample complexity lower bound for the Borda-winner identification problem in stochastic dueling bandits. 
In contrast, our lower bound which is similar in magnitude, applies to the regret which is always smaller (and often strictly smaller) than the sample complexity.



%% file: problem.tex
\section{Problem Setup} 
\label{sec:prob}

We consider an online decision task over a finite set of items $[K] :=
\{1,2,\ldots, K\}$ which spans over $T$ decision rounds. 
Initially, and obliviously, the environment fixes a sequence of $T$ \emph{preference matrices} $P_1,\ldots,P_T$, where each $\P_t \in [0,1]^{K \times K}$ satisfies $\P_t(i,j) = 1 - \P_t(j,i)$, and $\P_t(i,i) = \tfrac12$ for all $i,j \in [K]$.
The value of $\P_t(i,j)$ is interpreted as the probability that item $i$ wins when matched against item $j$ at time $t$. 
%
Then, at each round~$t$ the learner selects, possibly at random, two items $x_t,y_t\in [K]$
%
and a feedback $o_t \sim \operatorname{Ber}(\P_t(x_t,y_t))$ for the selected pair is revealed, where $\operatorname{Ber}(p)$ denotes a Bernoulli random variable with parameter $p$. Here, feedback of $o_t = 1$ implies that item $x_t$ wins the duel, while $o_t = 0$ corresponds to $y_t$ being the winner. 

The \emph{Borda score} of item $i \in [K]$ with respect to the preference matrix $P_t$ at time $t$ is defined as
\vspace{-5pt}
\begin{align*}
    \forall ~ i \in [K]~: \qquad
    b_t(i) := \frac{1}{K-1}\sum_{j \neq i} \P_t(i,j),  
    \qquad\text{and}\qquad
    i^* := \argmax_{i \in [K]} \sum_{t=1}^T b_t(i).
\end{align*}
i.e., $i^*$ is
the item  with the highest cumulative Borda score at time~$T$.
The learner's $T$-round regret $R_T$ is then defined as follows: 
\begin{align} \label{eq:regret}
    R_T 
    := \sum_{t=1}^T r_t ~,
    \qquad\text{where}\qquad
    r_t := b_t(i^*) - \tfrac12 (b_t(x_t)+b_t(y_t))
    .
\end{align}
We will consider two settings of preference assignments. 
In the \emph{general adversarial setting}, $P_1,\ldots,P_T$ is an arbitrary sequence of preference matrices.
In the \emph{fixed-gap setting}, preferences are set so that there is an item $i^* \in [K]$ for which, at all rounds $t \in [T]$, we have $\bar b_t(i^*) \geq \bar b_t(j) + \Delta$ for any other~$j \neq i^*$, where $\bar b_t(j) := \frac{1}{t}\sum_{\tau = 1}^t b_\tau(j)$ is the average Borda score of item $j \in [K]$ up to time $t$.

%% file: alg_borda.tex
\section{{General}~\prob}
\label{sec:alg_borda}

We first consider the general adversarial setup for an arbitrary sequence of preference matrices. We give an algorithm, called Dueling-EXP3 (D-EXP3), which has an expected regret of $\smash{ O((K\log K)^{1/3}T^{2/3}) }$. 
We also show how a simple modification of the D-EXP3 algorithm guarantees regret $\smash{ \tilde O(K^{1/3}T^{2/3}\sqrt{\log (K/\delta)}) }$ with probability at least $1-\delta$.

\subsection{The Dueling-EXP3 Algorithm}

Our algorithm, detailed in \cref{alg:dexp3}, is motivated from the classical EXP3 algorithm for adversarial MAB~\cite{auer02}, and relies on constructing unbiased estimates for scores of individual items at all rounds. 
However, in the dueling setup one has to establish such estimates using only binary preference feedback corresponding to a choice of a pair of items. Technically, the algorithm will estimate a \emph{shifted} version of the Borda score, defined as follows.
%
\begin{defn}
\label{rem:regeqv}
The \emph{shifted Borda score} of item $i \in [K]$ at time $t \in [T]$ is $s_t(i) := \tfrac1K \sum_{j \in [K]} \P_t(i,j)$. 
The \emph{shifted regret} is then defined as $R_T^s := \sum_{t = 1}^T [ s_t(i^*) - \tfrac12 ({s_t(x_t)+s_t(y_t)})]$. 
\end{defn}
Since all scored are ``shifted'' by the same value, this will not have any impact and the differences between Borda scores will be maintained (albeit multiplied by $\tfrac{K}{K-1}$). In particular, the best item is unchanged, i.e., $i^* = \argmax_{i \in [K]}\sum_{t = 1}^{T}b_t(i) = \argmax_{i \in [K]}\sum_{t = 1}^{T}s_t(i)$, and for any $K\ge 2$ and $T > 0$ we have $R_T = \tfrac{K}{K-1} R_T^s$.

At every round $t$, D-EXP3 maintains a weight distribution $\q_t \in \Delta{[K]}$ ($\Delta{[K]}$ is the $K$-simplex), and compute a score estimate $\ts_t(i)$ for each item $i$, being an unbiased estimate of $s_t(i)$ (\cref{lem:est_borda}). Thus, the cumulative estimated score $\sum_{\tau = 1}^t \ts_t(i)$ can be seen as the estimated \emph{cumulative reward} of item $i$ at round $t$, and hence $\q_{t+1}$ is simply updated running an exponential weight update on these estimated cumulative scores along with an $\gamma$-uniform exploration.

\begin{figure}[t]
\begin{center}
\vspace{-5pt}
\begin{algorithm}[H]
   \caption{\textbf{\algdexp \, (D-EXP3)} }
   \label{alg:dexp3}
\begin{algorithmic}[1]
   \STATE {\bfseries Input:} Item set indexed by $[K]$, learning rate $\eta > 0$, parameters $\gamma \in (0,1)$ 
   \STATE {\bfseries Initialize:} Initial probability distribution $\q_1(i) = 1/K, ~\forall i \in [K]$
   \FOR {$t = 1, \ldots, T$}
   \STATE Sample $x_t, y_t \sim \q_t$ i.i.d.~(with replacement) 
   \STATE Receive preference $o_{t}(x_t,y_t) \sim \text{Ber}(\P_t(x_t,y_t))$
   \STATE Estimate scores, for all $i \in [K]$: 
   \begin{align*}
       \ts_t(i) = 
       \frac{\1(x_t = i)}{K \q_t(i)}\sum_{j \in [K]}\frac{\1(y_t = j)o_t(x_t,y_t)}{\q_t(j)}
   \end{align*}
   \STATE Update, for all $i \in [K]$: 
   \begin{align*}
       \tq_{t+1}(i) 
       = 
       \dfrac{\exp(\eta \sum_{\tau = 1}^t \ts_\tau(i))}{\sum_{j = 1}^{K} \exp(\eta \sum_{\tau = 1}^t \ts_\tau(j))}
        \quad;\quad
       q_{t+1}(i) 
       = 
       (1-\gamma)\tq_{t+1}(i) + \frac{\gamma}{K}
   \end{align*}
   \ENDFOR
\end{algorithmic}
\end{algorithm}
\vspace{-12pt}
\end{center}
\end{figure}

We now state the expected regret guarantee we establish for \cref{alg:dexp3}.

\begin{restatable}{thm}{ubdexp}
\label{thm:ub_dexp}
Let $\eta = ((\log K)/(T\sqrt{K}))^{2/3}$ and $\gamma = \sqrt{\eta K}$. 
For any $T$, the expected regret of \cref{alg:dexp3} satisfies
$ 
 \E[R_T] \le 6(K\log K)^{1/3}T^{2/3}.
$ 
\end{restatable}

The proof of the expected regret bound crucially relies on the the following key lemmas regarding the estimates for the shifted Borda scores. We bound their magnitude, show that they are unbiased estimates, bound their instantaneous regret, and bound their second moment.

We first bound the magnitude of the estimates $\ts_t(i)$, using the fact that $q_t(j)\geq \gamma/K$.
\begin{restatable}{lem}{sbdd}
\label{lem:s01}
For all $t \in [T], i \in [K]$ it holds that $\ts_t(i) \leq K/\gamma^2$.
\end{restatable}
Next, we show that $\ts_t(i)$ is an unbiased estimate of the shifted Borda score $s_t(i)$. 
\begin{restatable}{lem}{estborda}
\label{lem:est_borda}
For all $t \in [T], i \in [K]$ it holds that  $\E[\ts_t(i)] = s_t(i)$.
\end{restatable}
Let $\cH_{t-1} := (\q_{1},\P_1,(x_1,y_1),o_1, \ldots \q_{t},\P_t)$ denotes the history up to time $t$.
We compute the expected instantaneous regret at time $t$ as a function of the true shifted Borsda scores at time $t$. 
\begin{restatable}{lem}{exploss}
\label{lem:exp_loss}
$\E_{\cH_t}[\q_t^\top\ts_t] = \E_{\cH_{t-1}}\big[\E_{x \sim \q_t} [s_t(x) \mid \cH_{t-1}]\big], \,\forall t \in [T]$.
\end{restatable}
Finally, We bound the second moment of our estimates.
\begin{restatable}{lem}{varborda}
At any time $t \in [T]$,
\label{lem:var_borda}
$
\E\big[ \sum_{i = 1}^K \q_t(i)\ts_t(i)^2  \big] \leq K/\gamma.
$
\end{restatable}


\noindent{\bf Proof overview.}
We upper bound $R_T^s$, the shifted Borda score regret, and recall that $R_T = \frac{K}{K-1} R^s_T$.
Note that  $\E_{\cH_T}[s_t(x_t)+s_t(y_t)] 
= \E_{\cH_{t-1}}\big[ \E_{x \sim \q_t}[2s_t(x) \mid \cH_{t-1}] \big]$, since $x_t$ and $y_t$ are i.i.d. Further note that we can write  
\begin{align*}
    \E_{\cH_T}[R_T^s] 
    = \E_{\cH_T}\!\!\left[ \sum_{t=1}^T [ s_t(i^*) - \tfrac12 (s_t(x_t)+s_t(y_t))] \right]
    \!=\! \max_{k \in [K]}\E_{\cH_T}\!\!\left[ \sum_{t=1}^T [s_t(k) - \tfrac12 ({s_t(x_t)+s_t(y_t)})] \right]
    \!\!,
\end{align*}
where the last equality holds since we assume the $\P_t$ are chosen obliviously and so $i^*$ does not depend on the learning algorithm. 
Thus we can rewrite:
\begin{align*}
    \E_{\cH_T}[R_T^s] 
    = \max_{k \in [K]}\left[ \sum_{t = 1}^T s_t(k) - \sum_{t = 1}^T \E_{\cH_{t-1}}[ \E_{x \sim \q_t}[s_t(x) \mid \cH_{t-1}] ] \right]
    .
\end{align*}
Now, as $\eta\ts_t(i) \le \eta K/\gamma^2$ (from \cref{lem:s01}), for any $\gamma \ge \sqrt{\eta K}$ and $\eta > 0$ we have $\eta\ts_t(i) \in [0,1]$.
From the regret guarantee of standard \emph{Exponential Weights} algorithm~\cite{auer02} over the completely observed fixed sequence of reward vectors $\ts_1, \ts_2, \ldots \ts_T$ we have for any $k \in [K]$:
\[
\sum_{t = 1}^T \ts_t(k) - \sum_{t = 1}^T \tq_t^\top \ts_t 
\leq 
\frac{\log K}{\eta} + \eta \sum_{t = 1}^T \sum_{i = 1}^K \tq_t(i)\ts_t(i)^2 
.
\]
%
Note that $\tq_t := (\q_t - \frac{\gamma}{K})/(1-\gamma)$. Let $i^* = \arg\max_{k \in [K]}\sum_{t = 1}^T s_t(k)=\arg\max_{k \in [K]}\sum_{t = 1}^T b_t(k)$. Taking expectation on both sides of the above inequality for $k= i^*$, we get:
\begin{align*}
\nonumber & (1-\gamma) \sum_{t = 1}^T \E_{\cH_T}[\ts_t(i^*)] - \sum_{t = 1}^T \E_{\cH_T}[\q_t^\top\ts_t] \le \frac{\log K}{\eta} + \E_{\cH_T}\bigg[\eta \sum_{t = 1}^T \sum_{i = 1}^K \q_t(i)\ts_t(i)^2\bigg],
\end{align*}
which by applying \cref{lem:est_borda}, \cref{lem:exp_loss} and \cref{lem:var_borda} and that $s_t(k^*) \le 1$, $\gamma = \sqrt{\eta K}$, we have
\begin{align*}
\E_{\cH_T}[R_T^s] \le 2T\sqrt{\eta K} + \frac{\log K}{\eta}
\quad\implies\quad
\E_{\cH_T}[R_T^s] \le 3(K\log K)^{1/3}T^{2/3},
\end{align*} 
where the implication follows by optimizing over $\eta$. The theorem  follows since $R_T = \tfrac{K}{K-1} R_T^s\leq 2R_T^s$. 
A complete proof is given in \cref{app:borda}.



%% file: alg_borda_whp.tex
\subsection{High Probability Regret Analysis}
\label{sec:alg_borda_whp}

We can show that a slightly modified version of \algdexp\, can lead to a high probability regret bound for the same setup. (This is inspired by the EXP3.P algorithm~\cite{AuerCFS02}.)
%
The modified algorithm runs almost identically to that of \cref{alg:dexp3}, except we now use a different score estimate $s'_t(i)$ in place of $\ts_t(i)$, where $s'_t(i)=\ts_t(i)+\beta/\q_t(i)$, where $\beta \in (0,1)$ is a tuning parameter. 
The items weights $\q_{t} \in \Delta{[K]}$ are now similarly updated using an exponential weight update on these modified score estimates along with an $\gamma$-uniform exploration. The complete algorithm is described in \cref{alg:dexp3_whp}.

\begin{center}
\begin{algorithm}[ht]
   \caption{\textbf{\algdexp ~(for High Probability Regret Bound)} }
   \label{alg:dexp3_whp}
\begin{algorithmic}[1]
   \STATE {\bfseries Input:} ~~ Item set: $[K]$, learning rate $\eta > 0$, parameters $\beta \in (0,1)$, $\gamma \in (0,1)$ 
   \STATE {\bfseries Initialize:} Initial probability distribution $\q_1(i) = \frac{1}{K}, ~\forall i \in [K]$
   \WHILE {$t = 1, 2, \ldots$}
   \STATE Sample $x_t, y_t \overset{iid}{\sim} \q_t$ (with replacement) 
   \STATE Receive preference $o_{t}(x_t,y_t) \sim \text{Ber}(\P_t(x_t,y_t))$
   \STATE Compute $\forall~i \in [K]$: 
   \begin{align*}
        s'_t(i) 
        = \frac{\1(x_t = i)}{K\q_t(i)}\sum_{j \in [K]}\frac{\1(y_t = j) o_t(x_t,y_t)}{\q_t(j)} + \frac{\beta}{\q_t(i)}
   \end{align*}
   \STATE Update $\forall~i \in [K]$: 
   $
       \tq_{t+1}(i) 
       = 
       \dfrac{\exp(\eta \sum_{\tau = 1}^t s'_\tau(i))}{\sum_{j = 1}^{K} \exp(\eta \sum_{\tau = 1}^t s'_\tau(j))}
        \quad;\quad
       q_{t+1}(i) 
       = 
       (1-\gamma)\tq_{t+1}(i) + \frac{\gamma}{K}
   $
   \ENDWHILE
\end{algorithmic}
\end{algorithm}
\vspace{-2pt}
\end{center}


We now prove a high probability regret bound for \cref{alg:dexp3_whp}:

\begin{restatable}{thm}{ubdexpwhp}
\label{thm:ub_dexp_whp}
Given any $T$ and $\delta>0$, there exists a setting of $\gamma$, $\beta$ and $\eta$, such that
with probability at least $1-\delta$, the regret of the modified D-EXP3 algorithm is
$
R_T = \Tilde{O}(K^{1/3}T^{2/3}) 
$,
\end{restatable}


The proof builds on the following steps.
Similarly to our estimates $\ts_t(i)$ above, we can show the following properties.


\begin{restatable}{lem}{sbddwhp}
\label{lem:s02}
For any item $i$ and round $t$, we have $s'_t(i) \leq K/\gamma^2 + K \beta/\gamma$.
\end{restatable}

\begin{restatable}{lem}{estbordawhp}
\label{lem:est_borda_whp}
For any item $i$ and round $t$, it holds that $\E[s'_t(i) \mid \cH_{t-1} ] = s_t(i) + \beta/\q_t(i)$.
\end{restatable}

However, unlike $\ts_t(i)$, the adjusted score estimates $s'_t(i)$ are \emph{no longer unbiased} for the true scores $s_t(i)$, and are larger in expectation by $\beta$.
Nevertheless, this does not hurt the regret analysis  
as its key element lies in showing that for any item $i \in [K]$, the cumulative estimated scores are not too far from the accumulated true scores. 
Precisely, the next lemma
ensures a high confidence upper bound on the cumulative scores $\sum_{t = 1}^Ts_t(i)$ and thus we can upper bound the learners performance in terms of estimated scores $s'_t$ (instead of $s_t$). 

\begin{restatable}{lem}{aconcestborda}
\label{lem:est_borda_acon}
For any $i \in [K]$, $\delta \in (0,1)$ and $\beta, \gamma \in (0,1)$, with probability at least $1-\delta$, we have
$$
\sum_{t = 1}^T s'_t(i) \ge \sum_{t = 1}^T s_t(i) - \frac{1}{\gamma\beta} \log\frac{1}{\delta}. 
$$
\end{restatable}

Incorporating this idea, the rest of the analysis closely follows that of \cref{thm:ub_dexp}.
%
%
See complete proof in \cref{appsub:borda_whp_proofs}.

%% file: alg_fxdgap.tex
\section{\gap~\prob}
\label{sec:alg_fxdgap}


In this section we study an adversarial setting with a fixed-gap of $\Delta>0$,
and give an algorithm with regret $O({(K\log (KT))}/{\Delta^2})$. 
%
In this case, our algorithm
is based on using confidence intervals of the \emph{estimated average Borda-scores}.
The algorithm has two phases. In the first phase, it samples uniformly at random two different items, and observes the outcome of their duel; in the second phase, it has a specific single item $\hat i$, which it uses in all rounds (for both items). 
The algorithm moves to its second phase when it detects an item $\hat i$ whose lower confidence bound ($LCB$) is larger than the upper confidence bound ($UCB$) of any other item $j$.
The complete description is given in \cref{alg:gap}.

Because of the non-stationary nature of the item preferences, and unlike classical action-elimination algorithms~\cite{Auer00,Even-DarMM06}, we still need to maintain an unbiased estimate of the Borda-score for every item at every round. 
(In contrast, in the stochastic dueling bandit problem~\cite{Zoghi+14RUCB}, for any fixed item $i \in [K]$, the unbiased estimate of its Borda score at round $t$ is also an unbiased estimate for any other round $s \neq t$; this simplifying condition does not hold in our fixed-gap adversarial model.)
Towards this, we maintain an estimate of the Borda score of any item $i \in [K]$ at any round $t$ as $\hb_{i}(t) := {K\1(x_t = i)}o_t(x_t,y_t)$, and show that it is an \emph{unbiased estimator}. 

\begin{restatable}{lem}{estbordagap}
\label{lem:est_borda_gap}
At any round $t$, we have $\E_{\cH_t}[\hb_{t}(i) ] = \b_{t}(i)$ for all $i \in [K]$.
\end{restatable}

Thus, an unbiased estimate for the $t$-step average Borda score $\bb_t(i)$, is $\tb_{t}(i) := \tfrac1t \sum_{\tau = 1}^t\hb_{\tau}(i)$.
We further maintain confidence intervals $[LCB(i;t),UCB(i;t)]$ around each $\tb_{t}(i)$, within which the means $\bb_i(t)$ lie with high probability.
\begin{restatable}{lem}{gapconf}
\label{lem:gap_conf}
With probability $\geq 1-\delta$, we have $\bb_i(t) \in [LCB(i;t),UCB(i;t)]$ for all $i$ and $t$.
\end{restatable}
The proof uses Bernstein's inequality to show that the estimates $\bb_i(t)$ are concentrated around their means $\tb_t(i)$, within the respective confidence intervals.
Assuming these confidence bounds hold, as soon as we find an item $\hat i \in [K]$ such that $LCB(\hat i;t) > UCB(j;t)$ for any other item~$j \neq \hat i$, we are guaranteed that $\hat i$ is the best item (in hindsight), i.e., $\hat i = i^*$. In the remaining rounds, $t+1,\ldots, T$, we play only  item $\hat i$ (for both items) and suffer no regret.
This results with the algorithm detailed in \cref{alg:gap}

\begin{restatable}{thm}{ubdgap}
\label{thm:ub_gap_whp}
Given any $\delta>0$, with probability at least $1-\delta$, the regret of \cref{alg:gap} (with parameter $\delta$) is upper bounded by
$
64 (K/\Delta^2) \log(2KT/\delta)
$.
\end{restatable}

We remark that unlike most MAB algorithms, we do not gain by incremental elimination. The reason is that we need to sample a second random item, $y_t$, which would have an expected Borda score which equals the average Borda score. This random item implies a constant regret per round until we identify $\hat i$. After we identify $\hat i$, with high probability, we do not incur any regret.

\vspace{-5pt}
\begin{figure}[H]
\begin{center}
\vspace{-10pt}
\begin{algorithm}[H]
   \caption{\textbf{\alggap \, (BCB)}}
   \label{alg:gap}
\begin{algorithmic}[1]
   \STATE {\bfseries Input:} item set indexed by $[K]$, confidence $\delta > 0$
   \FOR {$t = 1,\ldots,T$}
   \STATE Select $x_t, y_t \in [K]$, $x_t \neq y_t$ uniformly at random
   \STATE Receive preference $o_{t}(x_t,y_t) \sim \text{Ber}(\P_t(x_t,y_t))$
   \STATE Estimated score: $\hb_{i}(t) = K \, o_t(x_t,y_t) \, \1(x_t = i), ~\forall i \in [K]$		
   \STATE Estimated average score: $\tb_{t}(i) \leftarrow \frac{1}{t}\sum_{\tau = 1}^t\hb_{\tau}(i)$, $\forall i \in [K]$
  \STATE Compute:
  $
      LCB(i;t) = \tb_t(i) - 2\sqrt{\frac{K}{t}\log\frac{2KT}{\delta}}, \, 
      UCB(i;t) = \tb_t(i) + 2\sqrt{\frac{K}{t}\log\frac{2KT}{\delta}}
  $
   \STATE \textbf{if} $\exists ~ \hat i \in [K]$ s.t. $LCB(\hat i;t)> UCB(j;t)~~ \forall j \neq \hat i$, \textbf{then} break
   \ENDFOR
   \STATE Play $(\hat i,\hat i)$ for rest of the rounds $t+1, \ldots, T$. 
\end{algorithmic}
\end{algorithm}
\vspace{-10pt}
\end{center}
\end{figure}
%
\vspace{-10pt}


%% file: lower_bound.tex
\vspace{-10pt}

\section{Lower Bounds}
\label{sec:lb}

This section derives lower bounds for the adversarial dueling bandit settings. 
\cref{thm:lb_gap} and \cref{thm:lb} respectively give the regret lower bound for fixed gap and general adversarial setting. We first prove the following key lemma before proceeding to the individual lower bounds:

\begin{restatable}[]{lem}{lb}
\label{lem:lb}
For the problem of \prob\, with \obj\, objective, for any learning algorithm $\cA$ and any $\epsilon \in (0,0.1]$, there exists a problem instance (sequence of preference matrices $\P_1,\P_2,\ldots,\P_T$) such that the expected regret incurred by $\cA$ on that instance is at least  $\Omega(\min(\epsilon T, {K}/{\epsilon^2}))$, for any $K \ge 4$. 
\end{restatable}


\paragraph{Proof outline.}

The proof of the lemma has the following outline. We initially construct a stochastic preference matrix $P_0$, and later we consider perturbations of it. We start by describing $P_0$. We split the items to two equal size subsets $K_g$ and $K_b$. For any two items $i,j\in K_g$,  they are equally likely to win or lose in $P_0$, i.e., $P_0(i,j)=1/2$. Similarly, for any  $i,j\in K_b$ we have $P_0(i,j)=1/2$. When we pick item $i\in K_g$ and item $j\in K_b$ then item $i$ wins with probability $0.9$, i.e., $P_0(i,j)=0.9$. This implies that the Borda score of any $i\in K_g$ is $s(i)=0.7$ and for any $j\in K_b$ it is $s(j)=0.3$. 
Note that in $P_0$ all the items in $K_g$ have the highest Borda score. 

The main idea of the proof is that we will introduce a perturbation that will make one item $i^*\in K_g$ to have the highest Borda score. Formally, for each $i\in K_g$ we have a preference matrix $P_i$.
The only difference between $P_i$ and $P_0$ is in the entries of $i\in K_g$, where for any $j\in K_b$ we have $P_i(i,j)=0.9+\epsilon$.
We select our stochastic preference matrix at random from all the $P_i$ where $i\in K_g$, and denote by $i^*$ the selected index.
More explicitly following shows the form of $P_1$:
\[
\P_1 = \begin{bmatrix} 
0.5 & ... & 0.5 & 0.9+\epsilon & ... & 0.9+\epsilon\\
. & ... & . & . & ... & .\\
. & ... & . & . & ... & .\\
0.5 & ... & 0.5 & 0.9 & ... & 0.9\\
0.1-\epsilon & ... & 0.1 & 0.5 & ... & 0.5\\
. & ... & . & . & ... & .\\
. & ... & . & . & ... & .\\
0.1-\epsilon & ... & 0.1 & 0.5 & ... & 0.5\\
\end{bmatrix}
.
\]

A key observation is that in order to determine the best Borda score item, we need to match items $i\in K_g$ with items $j\in K_b$, since the expected outcome of other comparisons is known. However, each time we match an item $i\in K_g$ with an item $j\in K_b$ we have a constant regret of about $0.2-O(\epsilon)=\Theta(1)$. We will need to have $\Omega(|K_g|/\epsilon^2)$ samples to distinguish a bias of $\epsilon$ in the Borda score of $i^*\in K_g$ compared to other items $i\in K_g$. This leads to a regret of $\Omega(K/\epsilon^2)$. 
If, with some constant probability, we do not identify the item with the best Borda score, we will have a regret of  at least $\Omega(\epsilon T)$. This follows since any sub-optimal item has regret at least $\Omega(\epsilon)$ per time step.



We remark that the lower bound holds for $K=3$ with an almost an identical proof. (Technically, our lower bound requires that $K$ is even, but this is only for ease of presentation.) On the other hand, for $K=2$ the true regret bound scales $\Theta(1/\Delta)$, since when we match the (only) two items we have a regret of only $\Delta/2$.
Finally, there is an additional logarithmic dependency on the time horizon, which our lower bound does not capture.

\paragraph{Lower bound for the fixed-gap setting.}
In this case, given any fixed $\Delta > 0$, \cref{thm:lb_gap} shows a lower bound of $\Omega({K}/{\Delta^2})$.
The proof follows from \cref{lem:lb} setting $\epsilon = \Delta$.

\begin{restatable}{thm}{lbgap}
\label{thm:lb_gap}
Fix any $\Delta \in (0,0.1)$ and $K \ge 4$. For the fixed gap setting, for any learning algorithm $\cA$, there exists an instance with fixed gap $\Delta$, such that the expected regret incurred by $\cA$ on that instance is at least  $\Omega(\min(\Delta T, {K}/{\Delta^2}))$. 
%
\end{restatable}



The regret bound in this scales as $K/\Delta^2$ compared to $K/\Delta$ for MAB. The reason is that in order to distinguish between near-optimal items, the learner must compare them to significantly suboptimal items, which leads to the increase in the regret. Essentially, the regret bound is identical to the sample complexity bound in our lower bound instance.

\paragraph{Lower bound for the general adversarial setup.}
In this general case, since $\{\P_t\}_{t \in [T]}$ could be any arbitrary sequence, the adversary has the provision to tune $\epsilon$ based on $T$.
Precisely, given any $K$ and $T$, the adversary here can set $\epsilon = \Theta({K^{1/3}}/{ T^{1/3}})$. For any $T \ge K$ we guarantee that $\epsilon \in (0,0.1]$ and apply \cref{lem:lb}. For $T<K$ we clearly have a lower bound of $\Omega(T)$, since we need to sample each item at least once.
Therefore, for this general setup, we derive the following lower bound of $\Omega( K^{1/3}T^{2/3})$.

\begin{restatable}{thm}{lbgen}
\label{thm:lb}
For the problem of \prob\, with \obj\, objective, for any learning algorithm $\cA$, there exists a problem instance \acr$(K,T)$ with $T \ge K$, $K \ge 4$, and sequence of preference matrices $\P_1,\P_2,\ldots,\P_T$, such that the expected regret incurred by $\cA$ on that \acr$(K,T)$ is atleast $\Omega(K^{1/3}T^{2/3})$.
\end{restatable}



Note that the lower bound of $\Omega(T^{2/3})$ steams from the fact
that we can essentially cannot mix exploration and exploitation, at least in our lower bound instance.
Namely, while we are searching for the best Borda score item, we have a constant regret per time step. If we settle on any sub-optimal item, we get a regret of $\Omega(\epsilon T)=\Omega(T^{2/3})$, due to the selection of $\epsilon$.

%% file: conclusion.tex
\section{Conclusion and Future Scopes}
\label{sec:concl}

We considered the problem of dueling bandits with any adversarial preferences, i.e., \emph{adversarial dueling bandits}. To the best of our knowledge, this work is the first to consider the dueling bandit problem for fully adversarial setup. (The work of \cite{Adv_DB} introduced adversarial utility-based dueling bandits with a transfer function, which has very different characteristics, as we discussed earlier.)

We proposed algorithms for online regret minimization with Borda scores. We gave an  $\tilde{O}(K^{1/3}T^{2/3})$ regret algorithm (\algdexp\,) for the problem, and also shown optimality of our bounds with a matching $\Omega(K^{1/3}T^{2/3})$ lower bound analysis. 
We also proved a similar high probability regret bound.
Finally, for an intermediate \emph{fixed-gap} adversarial setup---which bridges the gap between stochastic and adversarial dueling bandits---we gave an $\smash{ \tilde{O}((K/\Delta^2)\log{T}) }$ regret algorithm, \alggap, and also a corresponding regret lower bound of $\Omega(K/\Delta^2)$.

Moving forward, one can potentially address many open threads along this direction; for example, considering other general notions of regret performances, considering the problem on larger (potentially infinite) arm-spaces, or even analyzing dynamic regret for adversarial preferences \cite{besbes19,luo17}. 
Few more open questions to answer here are: In case of more strcutured utility based preferences (e.g. Plackett-Luce preference model \cite{Az+12} etc.), where the item utility scores are chosen adversarially at every round, is it possible to show an improved performance limit of $\Theta(\sqrt{KT})$? In such cases, how does the learning rate varies with $K$ and $T$ for general subsetwise preferences (i.e. where more than two items can be compared at every round and the learner receives a winner feedback of the subset played) \cite{Brost+16,Ren+18}?
Another interesting direction would be to understand the connection of this problem with other bandit setups, e.g., learning with feedback graphs \cite{Alon+15,Alon+17} or other side information \cite{SideInfo11,SideInfo14}.

%% file: appendix.tex
	
	
\allowdisplaybreaks

\input{app-general}

\input{app-high-prob}

\input{app-fixed-gap}

\input{app-lower-bound}

%% file: app-general.tex
\section{Appendix for \cref{sec:alg_borda}}
\label{app:borda}

\sbdd*

\begin{proof}
The claim simply follows from the definition of $\ts_t(i)$, and the fact that $\q_t(i) \ge \frac{\gamma}{K}$, for all $i \in [K]$ and $t \in [T]$.
\end{proof}

\estborda*

\begin{proof}
Note that:
\begin{align*}
\E&\big[ \ts_t(i) \big] = \E_{\cH_t}\bigg[\frac{\1(x_t = i)}{\q_t(i) K}\sum_{j \in [K]}\frac{\big[\1(y_t = j)o_t\big]}{\q_t(j)}\bigg]\\
& = \frac{1}{K} \Bigg(\E_{\cH_t}\Bigg[ \sum_{j \in [K]} \frac{\1(x_t = i)\1(y_t = j)o_t}{\q_t(i)\q_t(j)} \Bigg] \Bigg)\\
& = \frac{1}{K} \Bigg(\E_{\cH_{t-1}}\Bigg[ \E_{(x_t,y_t,o_t)} \Big[ \frac{\1(x_t = i)}{\q_t(i)}\sum_{j \in [K]} \frac{\1(y_t = j)o_t}{\q_t(j)} \mid \cH_{t-1} \Big] \Bigg] \Bigg)\\
& = \frac{1}{K} \Bigg(\E_{\cH_{t-1}}\Bigg[ \E_{x_t}\bigg[ \frac{\1(x_t = i)}{\q_t(i)} \sum_{j \in [K]} \E_{y_t}\Big[ \frac{\1(y_t = j) \E_{o_t}\big[ o_t \mid y_t\big] }{\q_t(j)} \mid x_t \Big] \mid \cH_{t-1} \bigg] \Bigg]\Bigg)\\
& = \frac{1}{K} \Bigg(\E_{\cH_{t-1}}\Bigg[ \E_{x_t}\bigg[ \frac{\1(x_t = i)}{\q_t(i)} \sum_{j \in [K]} \E_{y_t}\Big[ \frac{\1(y_t = j) \P_t(x_t,y_t) }{\q_t(j)} \mid x_t \Big] \mid \cH_{t-1} \bigg] \Bigg]\Bigg)\\
& = \frac{1}{K} \Bigg(\E_{\cH_{t-1}}\Bigg[ \E_{x_t}\bigg[ \frac{\1(x_t = i)}{\q_t(i)} \sum_{j \in [K]} \sum_{j' = 1}^{K}\Big[ \frac{\1(j = j') \P_t(x_t,j')\q_t(j') }{\q_t(j)} \Big] \mid \cH_{t-1} \bigg] \Bigg] \Bigg)\\
& = \frac{1}{K} \Bigg(\E_{\cH_{t-1}}\Bigg[ \E_{x_t}\bigg[ \frac{\1(x_t = i)}{\q_t(i)} \sum_{j \in [K]} \P_t(x_t,j) \mid \cH_{t-1} \bigg] \Bigg] \Bigg)\\
& = \frac{1}{K} \Bigg(\E_{\cH_{t-1}}\Bigg[ \sum_{i' = 1}^{K}\bigg[ \frac{\1(i = i')\q_t(i')}{\q_t(i)} \sum_{j \in [K]} \P_t(i',j) \bigg] \Bigg] \Bigg)\\
& = \frac{1}{K} \Bigg( \sum_{j \in [K]\setminus\{i\}}\P_t(i,j) \Bigg) = \frac{1}{K}\sum_{j \in [K]\setminus\{i\}}\P_t(i,j) = s_t(i),
\end{align*}
which concludes the proof.
\end{proof}

\exploss*

\begin{proof}
Following the same techniques from the proof of \cref{lem:est_borda}, we have:
\begin{align*}
\E_{\cH_t}[\q_t^\top\ts_t] & = \E_{\cH_t}\Bigg[ \sum_{i = 1}^{K}\q_t(i)\ts_t(i) \Bigg] = \E_{\cH_{t-1}}\Bigg[\sum_{i = 1}^{K}\q_t(i)\E_{(x_t,y_t,o_t)}\Big[ \ts_t(i) \mid \cH_{t-1} \Big] \Bigg]\\
& \overset{(1)}{=} \E_{\cH_{t-1}}\Bigg[\sum_{i = 1}^{K}\q_t(i)s_t(i) \Bigg] = \E_{\cH_{t-1}}\big[\E_{x \sim \q_t} [s_t(x) \mid \cH_{t-1}]\big], 
\end{align*}
where $(1)$ follows from the proof of \cref{lem:est_borda}, and hence the result follows.
\end{proof}

\varborda*

\begin{proof}
Recall $\ts_t(i):=\frac{\1(x_t = i)}{\q_t(i) K}\sum_{j \in [K]}\frac{\big[\1(y_t = j)o_t\big]}{\q_t(j)}$. The argument follows similar to the proof of \cref{lem:est_borda}:

\begin{align*}
\E\big[ \sum_{i = 1}^K \q_t(i)\ts_t(i)^2  \big] &= \E_{\cH_{t-1}}\Bigg[ \sum_{i = 1}^{K}\q_t(i) \E_{(x_t,y_t,o_t)}\Bigg[ \sum_{j \in [K]} \frac{\1(x_t = i, y_t = j)o_t}{K\q_t(i)\q_t(j)} \mid \cH_{t-1} \Bigg]^2\Bigg]\Bigg)\\
& = \frac{1}{K^2}\Bigg(\E_{\cH_{t-1}}\Bigg[ \sum_{i = 1}^{K}\frac{\q_t(i)}{\q_t(i)^2} \E_{(x_t,y_t)}\Big[ \sum_{j \in [K]} \frac{1}{\q_t(j)^2}{\1(x_t = i)\1(y_t = j)\E_{o_t}\big[o_t^2 \mid (x_t,y_t) \big]} \mid \cH_{t-1} \Big]\Bigg] \Bigg)\\
& \le \frac{1}{K^2}\Bigg(\E_{\cH_{t-1}}\Bigg[ \sum_{i = 1}^{K}\frac{1}{\q_t(i)} \Big[ \sum_{j \in [K]}\frac{1}{\q_t(j)^2} \E_{x_t}\big[\1(x_t = i)\big]\E_{y_t}\big[ \1(y_t = j) \Big] \mid \cH_{t-1} \Big]\Bigg] \Bigg)\\
& = \frac{1}{K^2}\Bigg(\E_{\cH_{t-1}}\Bigg[ \sum_{i = 1}^{K}\frac{1}{\q_t(i)}\sum_{j \in [K]}\frac{1}{\q_t(j)^2}(\q_t(i)\q_t(j)) \Bigg] \Bigg)\\
& = \frac{1}{K^2}\Bigg(\E_{\cH_{t-1}}\Bigg[ \sum_{j = 1}^{K}\frac{K}{\q_t(k)}\Bigg] \Bigg)\\
& = \frac{1}{K}\Bigg(\E_{\cH_{t-1}}\Bigg[ \sum_{j = 1}^{K}\frac{1}{\q_t(j)} \Bigg] \Bigg) \le \frac{1}{K}\Bigg(\E_{\cH_{t-1}}\Bigg[ \sum_{j = 1}^{K}\frac{K}{\gamma} \Bigg] \Bigg) = \frac{K}{\gamma}, 
\end{align*}
where last inequality follows since $\q_t(j) \ge \frac{\gamma}{K}$, and the claim follows.
\end{proof}

\ubdexp*

\begin{proof}
Recall the definition of \emph{`shifted Borda score'} $s_t$ and the regret $R_T^s$ defined in terms of $s_t$. It would be convenient to  upper bound $R_T^s$ and recall that $R_T=(K/(K-1))R_T^s$.

Note that $\E_{\cH_T}[s_t(x_t)+s_t(y_t)] = \E_{\cH_{t-1}}\Big[ \E_{x_t, y_t \overset{\text{iid}}{\sim} \q_t}\big[ s_t(x_t)+s_t(y_t) \mid \cH_{t-1} \big] \Big] = \E_{\cH_{t-1}}\big[ \E_{x \sim \q_t}[2s_t(x) \mid \cH_{t-1}] \big]$, since $x_t$ and $y_t$ are i.i.d. Further note that we can write  
\[ \E_{\cH_T}[R_T^s] := \E_{\cH_T}\Bigg[\sum_{t = 1}^T \bigg[ s_t(i^*) - \frac{s_t(x_t)+s_t(y_t)}{2}\bigg] \Bigg] =  \max_{k \in [K]}\E_{\cH_T}\Bigg[\sum_{t = 1}^T \bigg[ s_t(k) - \frac{s_t(x_t)+s_t(y_t)}{2} \bigg]\Bigg],
\]
where the last equality holds since $\P_t$s are chosen obliviously, and hence $s_t$s and $i^*$ are independent of the randomness of the algorithm.
Thus we get: 

\begin{align}
\label{eq:preg1}
\E_{\cH_T}[R_T^s] = \max_{k \in [K]}\Big[ \sum_{t = 1}^T s_t(k) - \sum_{t = 1}^T \E_{\cH_{t-1}}\big[ \E_{x \sim \q_t}[s_t(x) \mid \cH_{t-1}] \big]\Big],
\end{align}

First, since $\eta\ts_t(i) \le \frac{\eta K}{\gamma^2}$ (from \cref{lem:s01}), for any $\gamma \ge \sqrt{\eta K}$ and $\eta > 0$, we have $\eta\ts_t(i) \in [0,1]$ for any $i \in [K], t \in [T]$.  From the regret guarantee of standard \emph{Exponential Weight} algorithm \cite{auer02} over the completely observed fixed sequence of reward vectors $\ts_1, \ts_2, \ldots \ts_T$ we have for any $k \in [K]$:
\[
\sum_{t = 1}^T \ts_t(k) - \sum_{t = 1}^T \big[\tq_t^\top \ts_t\big] \le \frac{\log K}{\eta} + \eta \sum_{t = 1}^T \sum_{i = 1}^K \tq_t(i)\ts_t(i)^2 
\]
where $\tq_t(i) := \dfrac{e^{\eta \sum_{\tau = 1}^{t-1}\ts_\tau(i)}}{\sum_{j = 1}^{K}e^{\eta \sum_{\tau = 1}^{t-1}\ts_\tau(j)}}, ~\forall i \in [K]$. 

Since $\tq_t = \frac{(\q_t - \frac{\gamma}{K})}{1-\gamma}$ and $\gamma \in (0,1)$, from above inequality we get for any $k \in [K]$: 
\[
(1-\gamma)\sum_{t = 1}^T \ts_t(k) - \sum_{t = 1}^T \q_t^\top \ts_t  \le \frac{\log K}{\eta} + \eta \sum_{t = 1}^T \sum_{i = 1}^K \q_t(i)\ts_t(i)^2. 
\]

Since $i^* = \arg\max_{k \in [K]}\sum_{t = 1}^T s_t(k)=\arg\max_{k \in [K]}\sum_{t = 1}^T b_t(k)$, using the above inequality for $k = i^*$, and taking expectation on both sides, we have

\begin{align}
\label{eq:preg2}
\nonumber & (1-\gamma) \sum_{t = 1}^T \E_{\cH_T}[\ts_t(i^*)] - \sum_{t = 1}^T \E_{\cH_T}[\q_t^\top\ts_t] \le \frac{\log K}{\eta} + \E_{\cH_T}\bigg[\eta \sum_{t = 1}^T \sum_{i = 1}^K \q_t(i)\ts_t(i)^2\bigg]\\
\nonumber \overset{(1)}{\implies}~ & (1-\gamma) \sum_{t = 1}^T s_t(i^*) - \sum_{t = 1}^T \E_{\cH_{t-1}}\big[ \E_{x \sim \q_t}[s_t(x) \mid \cH_{t-1}] \big]\Big] \le \frac{\log K}{\eta} + \eta \sum_{t = 1}^T \frac{K}{\gamma}\\
\nonumber \overset{(2)}{\implies}~ & \sum_{t = 1}^T s_t(i^*) - \sum_{t = 1}^T \E_{\cH_{t-1}}\big[ \E_{x \sim \q_t}[s_t(x) \mid \cH_{t-1}] \big]\Big] \le \gamma\sum_{t = 1}^T s_t(i^*) + \frac{\log K}{\eta} + \eta \sum_{t = 1}^T \frac{K}{\gamma}\\
\nonumber \overset{(3)}{\implies}~ & \E_{\cH_T}[R_T^s] \le \gamma T + \frac{\log K}{\eta} + \eta T \frac{K}{\gamma}\\
\nonumber \overset{(4)}{\implies}~ & \E_{\cH_T}[R_T^s] \le 2T\sqrt{\eta K} + \frac{\log K}{\eta}\\ 
\nonumber \overset{(5)}{\implies}~ & \E_{\cH_T}[R_T^s] \le 3(K\log K)^{1/3}T^{2/3}
\end{align} 
where $(1)$ follows from \cref{lem:est_borda}, \cref{lem:exp_loss} and \cref{lem:var_borda}, $(2)$ follows since $s_t(i^*) \le 1$, and $(3)$ follows from \cref{eq:preg1}, $(4)$ follows since $\gamma = \sqrt{\eta K}$, and $5$ follows by optimizing over $\eta$ which gives $\eta = \big(\frac{\log K}{T\sqrt{K}}\big)^{2/3}$.  
Further note that for $T \ge K \log K$, $\gamma = \sqrt{\eta K} \in [0,1]$ as desired. Finally from \cref{rem:regeqv} since $R_T = (K/(K-1))R_T^s$, this concludes the proof.
\end{proof}

%% file: app-high-prob.tex
\section{Appendix for \cref{sec:alg_borda_whp}}
\label{app:gapwhp}


\subsection{Proofs for \cref{sec:alg_borda_whp}}
\label{appsub:borda_whp_proofs}

\sbddwhp*

\begin{proof}
The claim simply follows from the definition of $s'_t(i)$, and the fact that $\q_t(i) \ge \frac{\gamma}{K}$, for all $i \in [K]$ and $t \in [T]$.
\end{proof}

\estbordawhp*

\begin{proof}
Note that for any $i \in [K]$,
\[
s'_t(i) = \frac{\frac{\1(x_t = i)}{K}\sum_{j \in [K]}\frac{\1(y_t = j) o_t}{\q_t(j)} + \beta}{\q_t(i)} = \frac{1}{K}\sum_{j \in [K]}\frac{\1(y_t = j)\1(x_t = i)o_t}{\q_t(i)\q_t(j)} + \frac{\beta}{\q_t(i)}
\]

Recalling the definition of $\ts_t(i):= \frac{1}{K}\sum_{j \in [K]}\frac{\1(y_t = j)\1(x_t = i)o_t}{\q_t(i)\q_t(j)}$ from \cref{alg:dexp3}, we further note:
\begin{align*}
\E\big[ s'_t(i) \mid \cH_{t-1} \big] &= \E_{(x_t,y_t,o_t)} \Bigg[ \ts_t(i) + \frac{\beta}{\q_t(i)} \mid \cH_{t-1} \Bigg]\\
& = s_t(i) +  \E_{(x_t,y_t,o_t)}\Bigg[ \frac{\beta}{\q_t(i)} \mid \cH_{t-1} \Bigg] ~~~(\text{from \cref{lem:est_borda}} )\\
& = s_t(i) +  \frac{\beta}{\q_t(i)} ~~~(\text{since $\q_t(i)$ is $\cH_{t-1}$ measurable}),
\end{align*}
which proves the claim.
\end{proof}

\aconcestborda*

\begin{proof}
Let $\beta' = \gamma \beta $. Then note $\beta' \in (0,1)$ by the choice of $\beta, \gamma$.
Thus using Markov Inequality:

\begin{align}
\label{eq:lem_aconc1}
\nonumber & Pr\Bigg( \sum_{t = 1}^Ts'_t(i) \le \sum_{t = 1}^Ts_t(i) - \frac{\log(1/\delta)}{\beta'} \Bigg) = Pr_{\cH_T}\Bigg( \text{exp}\bigg( \beta'\sum_{t = 1}^T\Big(    s_t(i) - s'_t(i) \Big) \bigg) \ge \frac{1}{\delta} \Bigg)\\
\nonumber & = \delta\E_{\cH_T}\bigg[ \text{exp}\bigg( \beta'\sum_{t = 1}^T\Big( s_t(i) - s'_t(i) \Big) \bigg) \bigg] 
= \delta\E_{\cH_T}\bigg[ \Pi_{t = 1}^T\text{exp}\bigg( \beta'\Big( s_t(i) - s'_t(i) \Big) \bigg) \bigg]\\
& = \delta\Pi_{t = 1}^T\E_{\cH_t}\bigg[ \text{exp}\bigg( \beta'\Big( s_t(i) - s'_t(i) \Big) \bigg) \mid \cH_{t-1} \bigg]
\end{align}

Now for any fixed $t \in [T]$, for any $i \in [K]$, note that $s'_t(i) \ge \frac{\beta}{\q_t(i)}$ (due to \cref{lem:est_borda_whp}). Thus since $s_t(i) \in [0,1]$ by definition, we have $s_t(i) - (s'_t(i)-\frac{\beta}{\q_t(i)}) \le 1$. Moreover as $\beta' \in (0,1)$, using $e^x \le (1+x+x^2)$ for any $x \le 1$, we get: 

\begin{align*}
& \E_{\cH_t}\bigg[ \text{exp}\bigg( \beta'\Big( s_t(i) - s'_t(i) \Big) \bigg) \mid \cH_{t-1} \bigg] = \E_{\cH_t}\bigg[ \text{exp}\bigg( \beta'\Big( s_t(i) - s'_t(i) + \frac{\beta}{\q_t(i)}\Big) \bigg) \mid \cH_{t-1} \bigg]\text{exp}\Big( - \frac{\beta \beta'}{\q_t(i)}\Big)\\
& \le \E_{\cH_t}\bigg[ 1 + \beta'\Big( s_t(i) - s'_t(i) + \frac{\beta}{\q_t(i)}\Big) + \beta^{'2}\Big( s_t(i) - s'_t(i) + \frac{\beta}{\q_t(i)}\Big)^2  \mid \cH_{t-1} \bigg]\text{exp}\Big( - \frac{\beta \beta'}{\q_t(i)}\Big)\\
& = \E_{\cH_t}\bigg[ 1 + \beta^{'2}\Big( s_t(i) - s'_t(i) + \frac{\beta}{\q_t(i)}\Big)^2  \mid \cH_{t-1} \bigg]\text{exp}\Big( - \frac{\beta \beta'}{\q_t(i)}\Big), \text{ \bigg(as } \E_{\cH_t}\Big[ s_t(i) - s'_t(i) + \frac{\beta}{\q_t(i)} \mid \cH_{t-1}\Big] = 0 \bigg) \\
& = 1 + \E_{\cH_t}\bigg[ \beta^{'2}\Big( s_t(i) - s'_t(i) + \frac{\beta}{\q_t(i)}\Big)^2  \mid \cH_{t-1} \bigg]\text{exp}\Big( - \frac{\beta \beta'}{\q_t(i)}\Big), ~\Big(\text{as } \q_t(i) \text{ is $\cH_t$ measurable} \Big)\\
& = 1 + \bigg[ \beta^{'2}\text{Var}_{\cH_t}\Big(s'_t(i) \mid \cH_{t-1} \Big) \bigg]\text{exp}\Big( - \frac{\beta \beta'}{\q_t(i)}\Big)\\
& \le 1 + \beta^{'2}\Bigg[ \E_{\cH_t}\bigg[ \Big( \frac{\1(x_t = i)}{\q_t(i)K}\sum_{j =1}^K\frac{\1(y_t = j)o_t}{\q_t(j)} \Big)^2 \mid \cH_{t-1} \bigg] \Bigg]\text{exp}\Big( - \frac{\beta \beta'}{\q_t(i)}\Big) ~~\Big(\text{since } \frac{\beta}{\q_t(i)} \text{ is constant given } \cH_{t-1} \Big)\\
& \le 1 + \beta^{'2}\Bigg[ \E_{\cH_t}\bigg[ \Big( \frac{\1(x_t = i)}{\q_t(i)^2K^2}\sum_{j =1}^K\frac{\1(y_t = j)}{\q_t(j)^2} \Big) \mid \cH_{t-1} \bigg] \Bigg]\text{exp}\Big( - \frac{\beta \beta'}{\q_t(i)}\Big) ~~\Big(\text{since } o_t \le 1\Big)\\
& \le 1 + \beta^{'2}\Bigg[ \bigg[ \Big( \frac{1}{\q_t(i)K^2}\sum_{j =1}^K\frac{1}{\q_t(j)} \Big) \bigg] \Bigg]\text{exp}\Big( - \frac{\beta \beta'}{\q_t(i)}\Big)\\
& \le 1 + \Big( \beta'\frac{\beta'}{\q_t(i)\gamma} \Big)\text{exp}\Big( - \frac{\beta \beta'}{\q_t(i)}\Big) ~~\Big(\text{since } \q_t(j) \ge \frac{\gamma}{K},~ \forall j \in [K]\Big)\\
& \le \Big(1 + \frac{\beta \beta'}{\q_t(i)} \Big)\text{exp}\Big( - \frac{\beta \beta'}{\q_t(i)}\Big) ~~\Big(\text{since } \beta = \frac{\beta'}{\gamma} \Big)\\
& \le 1 ~~\Big(\text{since } (1+x) \le e^x \text{ for any } x \in \R \Big)
\end{align*}
Applying the above result for all $t \in [T]$ in \cref{eq:lem_aconc1} concludes the proof.
\end{proof}

\ubdexpwhp*

\begin{proof}
We set $\gamma = \sqrt{2\eta K}$, $\beta = \frac{T^{-1/2}\sqrt{\log (K/\delta)}}{(2\eta)^{1/4}K^{3/4}}$, and $\eta = \big(\frac{\log K}{T\sqrt{2K}}\big)^{2/3}$.
We will prove a regret bound of 
\[
 R_T \le 2\bigg(3(2\log K)^{1/3} + 2^{5/6}\frac{\sqrt{ \log K / \delta}}{(\log K)^{1/6}}\bigg)K^{1/3} T^{2/3} = \tilde{O} (K^{1/3}T^{2/3}).
\]

Note that for $T<2T\log T$, the bound regret bound is more than $T$ and therefore holds trivially. For the remainder of the proof we assume that $T\geq 2T\log T$.

We start by recalling the definition of \emph{`shifted borda score'} $s_t$ and the regret $R_T^s$ from \cref{sec:prob}. Same as \cref{thm:ub_dexp}, we find it convenient to first upper bound $R_T^s$. 

Note that by \cref{lem:s02}, if we set $\eta \le \Big( \frac{K}{\gamma^2} + \frac{K \beta}{\gamma} \Big)^{-1}$, we have $\eta s'_t(i) \le \eta\Big( \frac{K}{\gamma^2} + \frac{K \beta}{\gamma} \Big) \in (0,1), ~\forall i \in [K], t \in [T]$. Then again  from the regret guarantee of standard \emph{Exponential Weight} algorithm \cite{auer02} over the fully observed fixed sequence of reward vectors $s'_1, s'_2, \ldots s'_T$ we have for any $k \in [K]$:

\[
\sum_{t = 1}^T s'_t(k) - \sum_{t = 1}^T \big[\tq_t^\top s'_t\big] \le \frac{\log K}{\eta} + \eta \sum_{t = 1}^T \sum_{i = 1}^K \tq_t(i)s'_t(i)^2 
\]
where $\tq_t(i) := \dfrac{e^{\eta \sum_{\tau = 1}^{t-1}s'_\tau(i)}}{\sum_{j = 1}^{K}e^{\eta \sum_{\tau = 1}^{t-1}s'_\tau(j)}}, ~\forall i \in [K]$. 
Further by definition since $\tq_t = \frac{(\q_t - \frac{\gamma}{K})}{1-\gamma}$ and $\gamma \in (0,1)$, from above inequality we get for any $k \in [K]$: 

\begin{align}
\label{eq:thmwhp_1}
(1-\gamma) \sum_{t = 1}^T s'_t(k) \le \sum_{t = 1}^T \q_t^\top s'_t + \frac{\log K}{\eta} + \eta \sum_{t = 1}^T \sum_{i = 1}^K \q_t(i)s'_t(i)^2,
\end{align}

Here note that for a given $x_t$, 
\begin{align}
\label{eq:thmwhp_1.5}
\q_t^\top s'_t = \sum_{i = 1}^K\q_t(i)\Bigg( \frac{\frac{\1(x_t = i)}{K}\sum_{j \in [K]}\frac{\1(y_t = j) o_t(i,j)}{\q_t(j)} + \beta}{\q_t(i)} \Bigg) = \Bigg(\sum_{j \in [K]}\frac{\1(y_t = j) o_t(x_t,j)}{\q_t(j)K} + K\beta\Bigg),
\end{align}
where $o_t(i,j) \sim \text{Ber}\big(\P_t(i,j)\big), ~\forall i,j \in [K]$. Thus taking expectation:

\begin{align}
\label{eq:thmwhp_2}
\E_{y_t,o_t}\bigg[ \q_t^\top s'_t \mid \cH_{t-1}, x_t\bigg]
= \E_{y_t,o_t}\bigg[\sum_{j \in [K]}\frac{\1(y_t = j) o_t(x_t,j)}{\q_t(j)K} \mid \cH_{t-1}, x_t\bigg] + K \beta = s_t(x_t) + K \beta
\end{align}

Further noting $\q_t(i)s'_t(i) \le \Big(\beta + \gamma^{-1} \Big)$,

\begin{align}
\label{eq:thmwhp_3}
\eta \sum_{t = 1}^T\sum_{i = 1}^K \q_t(i)s'_t(i)^2 \le \eta \sum_{t = 1}^T\Big(\beta + \gamma^{-1} \Big)\sum_{i = 1}^K s'_t(i) \le \eta K\Big(\beta + \gamma^{-1} \Big)\sum_{t = 1}^T s'_t(i^*),
\end{align}

where we denote by $i^* = \arg\max_{k \in [K]}\sum_{t = 1}^T s'_t(k)$. Combining the results of \cref{eq:thmwhp_3} to \cref{eq:thmwhp_1}, and the fact that we set $\eta \le \Big( \frac{K}{\gamma^2} + \frac{K \beta}{\gamma} \Big)^{-1}$,

\begin{align*}
& (1-\gamma)\sum_{t = 1}^T s'_t(i^*) \le \sum_{t = 1}^T \q_t^\top s'_t + \frac{\log K}{\eta} + \gamma\sum_{t = 1}^T s'_t(i^*) ~\Bigg( \text{since } \eta \le \Big( \frac{K}{\gamma^2} + \frac{K \beta}{\gamma} \Big)^{-1} \Bigg)\\ 
\implies & (1-2\gamma)\bigg[\sum_{t = 1}^T s'_t(i^*)\bigg] \le \sum_{t = 1}^T \q_t^\top s'_t + \frac{\log K}{\eta} \\
\overset{(1)}{\implies} & (1-2\gamma)\max_{k \in [K]}\bigg[ \sum_{t = 1}^Ts_t(k) - \frac{\log(K/\delta)}{\gamma \beta} \bigg] \le \sum_{t = 1}^T \bigg[ \sum_{j \in [K]}\frac{\1(y_t = j) o_t(x_t,j)}{\q_t(j)K} + K\beta\bigg] + \frac{\log K}{\eta}\\
\overset{(2)}{\implies} & (1-2\gamma)\max_{k \in [K]}\bigg[ \sum_{t = 1}^Ts_t(k) - \frac{\log(K/\delta)}{\gamma \beta} \bigg] \le \sum_{t = 1}^T \bigg[ s_t(x_t) + K\beta\bigg] + \frac{\log K}{\eta}\\
\implies & \max_{k \in [K]}\bigg[ \sum_{t = 1}^Ts_t(k)  \bigg] - \sum_{t = 1}^T s_t(x_t) \le 2\gamma \max_{k \in [K]}\bigg[ \sum_{t = 1}^Ts_t(k)  \bigg] + K\beta T + \frac{\log K}{\eta} + (1-2\gamma)\frac{\log(K/\delta)}{\gamma \beta}\\
\implies & \max_{k \in [K]} \sum_{t = 1}^Ts_t(k) - \sum_{t = 1}^T s_t(x_t) \le 2\gamma T + K\beta T + \frac{\log K}{\eta} + \frac{\log(K/\delta)}{\gamma \beta} ~~(\text{since } \max_{k \in [K]} \sum_{t = 1}^Ts_t(k) \le T)\\
\overset{(3)}{\implies} & \max_{k \in [K]} \sum_{t = 1}^Ts_t(k) - \sum_{t = 1}^T s_t(x_t) \le 2\sqrt{2\eta K} T + K\beta T + \frac{\log K}{\eta} + \frac{\log(K/\delta)}{\beta \sqrt{2\eta K}}\\
\overset{(4)}{\implies} & \max_{k \in [K]} \sum_{t = 1}^Ts_t(k) - \sum_{t = 1}^T s_t(x_t) \le 2\sqrt{2\eta K} T + \frac{\log K}{\eta} + \frac{2^{3/4} K^{1/4}\sqrt{\log (K/\delta)}\sqrt{T}}{\eta^{1/4}}\\
\overset{(5)}{\implies} & \max_{k \in [K]} \sum_{t = 1}^Ts_t(k) - \sum_{t = 1}^T s_t(x_t) \le 3(2K T^2 \log K)^{1/3} + 2^{5/6}(K T^2)^{1/3}\frac{\sqrt{ \log K / \delta}}{(\log K)^{1/6}}\\
\implies & R_T^s \le \bigg(3(2\log K)^{1/3} + 2^{5/6}\frac{\sqrt{ \log K / \delta}}{(\log K)^{1/6}}\bigg)K^{1/3} T^{2/3},   
\end{align*} 
where $(1)$ follows from \cref{eq:thmwhp_1.5} and taking an union bound over all $i \in [K]$ for the claim \cref{lem:est_borda_acon}, $(2)$ holds from \cref{eq:thmwhp_2}, $(3)$ follows by setting $\gamma = \sqrt{2 \eta K}$ (note since if we can ensure $\beta \gamma \le 1$, this ensures $\eta \Big( \frac{K}{\gamma^2} + \frac{K \beta}{\gamma} \Big) \le \Big( \frac{2K\eta}{\gamma^2} \Big) \le 1$ as desired), $(4)$ follows by setting $\beta = \frac{\sqrt{\log (K/\delta)}}{(2\eta)^{1/4}K^{3/4}\sqrt T}$, $(5)$ follows by optimizing over $\eta$ which gives $\eta = \big(\frac{\log K}{T\sqrt{2K}}\big)^{2/3}$.  
Further note that any $T \ge 2K \log K$ implies $\gamma = \sqrt{2 \eta K} \in [0,1]$, and any $T \ge \frac{\log (K/ \delta)^{3/2}}{K^2 \sqrt{2 \log K}}$ implies $\beta \in (0,1)$ as desired. Finally from \cref{rem:regeqv} since $R_T \le 2R_T^s$, this concludes the proof.
\end{proof}

%% file: app-fixed-gap.tex
\section{Appendix for \cref{sec:alg_fxdgap}}
\label{app:gap}

\estbordagap*

\begin{proof}
Note that by definition for any $i \in [K]$, $\hb_{t}(i) = K\1(x_t = i)\sum_{j = 1}^{K} \1(y_t = j)o_t(i,j)$.
It is easy to see that for any $i \in [K]$, $t \in [T]$,
\begin{align*}
\E_{\cH_t}&\big[ \bb_{t}(i) \big] = \E_{\cH_{t-1}}\bigg[\E_{x_t,y_t,o_t}\bigg[ K\1(x_t = i)\sum_{j = 1}^{K} \1(y_t = j)o_t(i,j) \mid \cH_{t-1} \bigg]\bigg] = \E_{\cH_{t-1}}\big[b_t(i)\big] = b_t(i),
\end{align*}
where the last equality is simply due to the fact that $\P_t$ is chosen obliviously w.r.t. the history $\cH_{t-1}$, and the second last equality follows since for any $i \in K$:

\begin{align*}
E_{x_t,y_t,o_t}&\bigg[ K\1(x_t = i)\sum_{j = 1}^{K} \1(y_t = j)o_t(i,j) \mid \cH_{t-1} \bigg]\\ 
& = \E_{x_t}\bigg[K\1(x_t = i)\E_{y_t}\Big[ \sum_{j = 1}^K\1(y_t = j)\E_{o_t}[o_t(x_t,y_t)\mid y_t] \mid x_t \Big] \mid \cH_{t-1} \bigg]\\
& =  \E_{x_t}\bigg[K\1(x_t = i)\sum_{j = 1}^K\E_{y_t}\big[\1(y_t = j) \P_t(x_t,y_t)] \mid x_t \big] \mid \cH_{t-1} \bigg]\\
& = \E_{x_t}\bigg[\frac{K\1(x_t = i)}{K-1}\sum_{j = 1}^K \P_t(x_t,j)  \bigg] = \E_{x_t}[K\1(x_t = i) b_{t}(x_t)] = b_{t}(i)
,
\end{align*}
which concludes the proof.
\end{proof}

\gapconf*

\begin{proof}
For any round $t\leq 4K\log (2KT/\delta)$ we have that $2\sqrt{(K/t)\log(2KT/\delta)}\geq 1$ and therefore, for any item $i\in[K]$ we have $LCB(i,t)<0$ and $UCB(i;t>1$, and hence the lemma holds trivially.

Let us fix any item $i \in K$, and some $t \geq 4K \log \frac{K T}{\delta}\}$. Note that owning to our `random arm-pair $(x_t,y_t)$ selection strategy', the  random variables $\hb_1(i),\hb_2(i),\ldots,\hb_t(i)$ are independent. Let us denote denote by $\epsilon = \sqrt{\frac{4 K\log (2KT/\delta)}{t}}$. 
Let us also define for any $\tau \in [t]$, $z_t(\tau) = \frac{1}{t}(\hb_\tau(i)- b_\tau(i))$. Then note: (i). $z_1(i),z_2(i),\ldots,z_t(i)$ are also independent (ii). $\E_{\cH_t}[z_t(\tau)] = 0$ (see \cref{lem:est_borda_gap}), (iii). $|z_t(\tau)| < \frac{K}{t}$, and (iv). $\sum_{\tau = 1}^{t}\E_{\cH_t}[z_\tau^2(i)] = \E_{\cH_{t-1}}\big[ \E_{x_t,y_t,o_t}[z_\tau^2(i)]\mid \cH_{t-1}\big] \le \frac{1}{K}\frac{K^2}{t^2} + \frac{K-1}{K}\frac{1}{t^2} \le \frac{K+1}{t^2}$ (as $Pr(x_t = i) = \frac{1}{K}$) for any $K \ge 2$. Hence applying Bernstein's inequality we get: 

\begin{align*}
Pr\bigg( |\sum_{\tau = 1}^{t} z_\tau(i)| &\ge \epsilon \bigg) \le 2\text{exp}\bigg( -\frac{\epsilon^2/2}{ \frac{K+1}{t} + \frac{\epsilon K}{3t} } \bigg) \\
& \le 2\text{exp}\bigg( -\frac{ t \epsilon^2}{ 4 K } \bigg) = 2\text{exp}\bigg( -\frac{\frac{4K t\log (2KT/\delta)}{t}}{4K}  \bigg) = \frac{\delta}{KT},
\end{align*}
where the second inequality follows since for any $t > \frac{16K \log (2K T/ \delta)}{9}$, we have $\frac{\epsilon K}{3t} < \frac{K-1}{t}$.
The proof follows taking union bound over all $i \in [K]$ and $t \in [T]$.
\end{proof}

\ubdgap*

\begin{proof}
We would first assume the good event of \cref{lem:gap_conf}: $\forall i \in K, \forall t\in [T]$, we have $\bb_i(t) \in [LCB(i;t),UCB(i;t)]$.

Recall that by problem setup: $~\exists i^* \in [K],\, \forall t \in [T]$ such that $\bar b_t(i^*) > \bar b_t(j) + \Delta, \forall j \in [K]\sm\{i^*\}$.
Then if $t > \frac{64K \log (2KT/\delta)}{\Delta^2}$, this implies $\sqrt{\frac{4K\log (2KT/\delta)}{t}} \le \Delta/4$. Thus for any $j \in K\sm\{i^*\}$, at any $t > \frac{64K \log (2KT/\delta)}{\Delta^2}$,  

\begin{align*}
UCB(j;t) &= \tb_t(j) + \sqrt{\frac{4K\log (2KT/\delta)}{t}} \le \bb_t(j) + 2\sqrt{\frac{4K\log (2KT/\delta)}{t}} \leq \bb_t(j) -\Delta/2
\end{align*}
On the other hand, for $i^*$ we have 
\begin{align*}
 \bb_t(i^*) - \Delta/2 < \bb_t(i^*) - \Delta + 2\sqrt{\frac{4K\log (2KT/\delta)}{t}} < \bb_t(i^*) - 2\sqrt{\frac{4K\log (2KT/\delta)}{t}} < LCB(i^*;j).
\end{align*}
Since $ \bb_t(i^*) \geq  \bb_t(j)+\Delta $, it implies that $UCB(j;t)<LCB(i^*;t)$ for $t > \frac{64K \log (2KT/\delta)}{\Delta^2}$.
Thus for any $t > \frac{64K \log (2KT/\delta)}{\Delta^2}$, the algorithm would detect $\hat i = \{i^*\}$, and hence the regret at $\tau$ is $r_\tau = 0$ for the remaining rounds $\tau= t+1,\ldots, T$.
The final high probability regret upper bound now follows from the statement of \cref{lem:gap_conf} and the fact that instantaneous regret at any round $t$ such that $(x_t,y_t) \neq (i^*,i^*)$ is at most $1$.
\end{proof}

%% file: app-lower-bound.tex
\section{Appendix for Sec.~\ref{sec:lb}}
\label{app:lb}

\lb*

\begin{proof}
We will show specifically that for $T \le \frac{K}{1440 \epsilon^3}$ we have $R_T =  \Omega\big(\epsilon T\big)$ and for $T > \frac{K}{1440 \epsilon^3}$ we have $R_T=\Omega\Big(\frac{K}{\epsilon^2}\Big)$.


The proof relies on constructing a `hard enough' problem instance for the learning framework and showing no algorithm can achieve a smaller rate of regret on that instance than the claimed lower bounds.

For simplicity of notation we assume $K$ is even (similar technique could also be used to prove the same bound when $K$ is odd, and show that the lemma also applies to $K=3$. We denote by $\tK:= \frac{K}{2}$. 
Let us construct $\tK+1$ problem instances $\cI^1,\cI^2,\ldots,\cI^\tK$ and $\cI^0$, where each instance is uniquely identified by its underlying preference matrix as defined below: 

\noindent \textbf{Problem instance}$(\cI^0)$: For all $t \in [T]$, $P_t(i,j) = 
\begin{cases} 
0.5,~\forall i,j \in [\tK]$ \text{ or } $i,j \in [K]\sm [\tK]\\
0.9,~\forall i \in [\tK]$ \text{ and } $\forall j \in [K]\sm [\tK]\\
\end{cases}
$,  

or more explicitly:
\[
\P_t = \begin{bmatrix} 
0.5 & ... & 0.5 & 0.9 & ... & 0.9\\
. & ... & . & . & ... & .\\
. & ... & . & . & ... & .\\
. & ... & . & . & ... & .\\
0.5 & ... & 0.5 & 0.9 & ... & 0.9\\
0.1 & ... & 0.1 & 0.5 & ... & 0.5\\
. & ... & . & . & ... & .\\
. & ... & . & . & ... & .\\
. & ... & . & . & ... & .\\
0.1 & ... & 0.1 & 0.5 & ... & 0.5\\
\end{bmatrix}, ~~\forall t \in [T].
\]

Note for $\cI^0$, $\forall t \in [T]$, for any item $i \in [\tK]$, $s_t(i) = 0.7$, and for any item $i \in [K]\sm[\tK]$, $s_t(i) = 0.3$. 
Thus for the instance $\cI^0$, any item $i \in [\tK]$ is an optimal arm. Now let us consider $\tK$ alternative problem instances $\cI^{m} ~~\forall m \in [\tK]$:

\noindent \textbf{Problem instance}$(\cI^{m})$: 
For all $t \in [T]$, $P_t(i,j) = 
\begin{cases} 
0.5,~\forall i,j \in [\tK]$ \text{ or } $i,j \in [K]\sm [\tK]\\
0.9,~\forall i \in [\tK]$ \text{ and } $\forall j \in [K]\sm [\tK]\\
0.9+\epsilon,~\text{if } i = m, \forall j \in [K]\sm [\tK]
\end{cases}
$,  
for some $\epsilon\in (0,0.1]$. For example $\cI^{1}$ would be:
\[
\P_t = \begin{bmatrix} 
0.5 & ... & 0.5 & 0.9+\epsilon & ... & 0.9+\epsilon\\
. & ... & . & . & ... & .\\
. & ... & . & . & ... & .\\
. & ... & . & . & ... & .\\
0.5 & ... & 0.5 & 0.9 & ... & 0.9\\
0.1-\epsilon & ... & 0.1 & 0.5 & ... & 0.5\\
. & ... & . & . & ... & .\\
. & ... & . & . & ... & .\\
. & ... & . & . & ... & .\\
0.1-\epsilon & ... & 0.1 & 0.5 & ... & 0.5\\
\end{bmatrix}, ~~\forall t \in [T],
\]
and so on.
Note for any $\cI^{m}$, $\forall t \in [T]$, for any item $s_t(i) = 
\begin{cases}
0.7 , ~\forall i \in [\tK]\sm\{m\}\\
0.7+\epsilon, \text{ if } i = m\\
0.3-\frac{\epsilon}{K}, \text{ if } i = n
\end{cases}
$. 

Clearly, for instance $\cI^{m}$, the unique \emph{`best'} item is $i^{m}_{*}:=m$, and all $j \in [K]\sm[\tK]$ items are the \emph{`bad'} playing which at any round $t \in [T]$ yields a  constant regret of $\frac{\big(s_t(i^{m}_{*}) - s_t(j)\big)}{2} = 0.2 + \frac{(K+1)\epsilon}{2K}$, and all $j \in [\tK]\sm\{m\}$ items are the \emph{`near-best'} playing which at any round $t \in [T]$ yields at least a regret of $\frac{\big(s_t(i^{m}_{*}) - s_t(j)\big)}{2} = \frac{\epsilon}{2}$. However in order to distinguish the \emph{`best'} and the \emph{`near-best'} items, it is necessary to play the \emph{`bad'} items `sufficiently enough' to infer which of the $[\tK]$ items has the highest borda score. Intuitively the main idea of our lower bound technique lies in showing that in this process any learner has to pull the \emph{`bad'} items \emph{at least} a certain number of times which would lead consequently lead to the regret lower bound. 
The remaining arguments proves this formally. 

Towards this let us first define a few notations: For any algorithm $\cA$, let $N^{\cA}_T(i,j) := \E[\sum_{t = 1}^T\1(\{i,j\} = \{x_t,y_t\})]$, denotes the expected number of times $\cA$ pulls arm-pair $(i,j) \in [K]\times[K]$ in $T$ rounds (the expectation is taken over the randomness of the preference feedback). 
For simplicity of notations, we henceforth would denote $N^{\cA}_T(\cdot) = N_T(\cdot)$. 
We also denote by $D_t = \{x_t,y_t\}$, and $\Delta^{m}_j = \frac{\big(s_t(i^m_*) - s_t(j)\big)}{2}$, for all $m \in [\tK],~ j \in [K]$.

We now make the following two key observations:

\textbf{Observation $1$.} We consider only the class of all deterministic algorithms, i.e. where $x_t,y_t$ is a deterministic function of the past history $\cH_{t-1}$. Note this is without loss of generality, since any randomized strategy can be seen as a randomization over deterministic querying strategies. Thus, a lower bound which holds uniformly for any deterministic class of algorithms, would also hold over a randomized class of algorithms.

\textbf{Observation $2$.} We also consider that for any instance $\cI^m$ ($m \in [\tK]\cup\{0\}$), the algorithm $\cA$ pulls the `suboptimal-pairs' (i.e. any pair which contains at least one bad arm from $[K]\sm[\tK]$) for at most $\epsilon T$ times, i.e. for any $m \in [\tK]\cup\{0\}$, $\E_{\cI^{m}}[\sum_{i,j \mid \{i,j\} \cap [K]\sm[\tK] \neq \emptyset} N^{\cA}_T(i,j)] \le \epsilon T$. This is without loss of generality, since otherwise we already have an $\Omega(\epsilon T)$ lower bound.
  
We now turn our attention to proving the main result. We will break it into the following two case analyses:
$(1).\, T \le \frac{K}{1440 \epsilon^3}$, and $(2).\, T > \frac{K}{1440 \epsilon^3}$. 

\textbf{Case 1. \big($T \le \frac{K}{1440 \epsilon^3}$\big):} 
Firstly recall from Rem. \ref{rem:regeqv} regret definition $R_T^s$ defined in terms of the \emph{`shifted-borda score'} $s_t$. It would be convenient to first lower bound $R_T^s$.
As argued earlier, since for any $m \in [\tK]$ and $j \neq i^{m}_{*} = m$, $\Delta^{m}_j \ge \frac{\epsilon}{2}$, note the regret of any algorithm $\cA$ on instance $\cI^{m}$ for $T$ rounds, can be lower bounded as: 

\begin{align*}
\nonumber \E_{\cI^{m}}[R_T^s(\cA)] &= \sum_{t = 1}^T\sum_{i = 1}^K\sum_{j = 1}^K\bigg(\E_{\cI^{m}}\big[\1(D_t = \{i,j\})\big]\frac{\Delta^{m}_i + \Delta^{m}_j}{2}\bigg) \\
& \ge \sum_{t = 1}^T\bigg(\E_{\cI^{m}}\big[\1(D_t \neq \{i^{m}_{*},i^{m}_{*}\})\big]\frac{\epsilon}{2}\bigg)\\
& \ge \sum_{t = 1}^T\frac{\epsilon}{2}\bigg(T - \E_{\cI^{m}}\big[\1(D_t = \{i^{m}_{*},i^{m}_{*}\})\big]\bigg) = \frac{\epsilon}{2}\big(T - \E_{\cI^{m}}[N_T(i^{m}_{*},i^{m}_{*})]\big). 
\end{align*}

Then taking average over $\cI^{m}$s for all $m \in [\tK]$: 
\begin{align}
\label{eq:lb0}
\E[R_T^s(\cA)] = \sum_{m \in [\tK]}\frac{\E_{\cI^{m}}[R_T^s(\cA)]}{\tK} \ge \frac{\epsilon}{2}\bigg(T - \frac{\sum_{m \in [\tK]}\E_{\cI^m}[N_T(m,m)]}{\tK}\bigg)
\end{align}
since $i^{m}_{*}=m$. 
Now note that:

\begin{align}
\label{eq:lb1}
\nonumber & \E_{\cI^{m}}[N_T(m,m)] - \E_{\cI^{0}}[N_T(m,m)] \\
& = \sum_{t = 1}^T\big(Pr_{\cI^{m}}(D_t=\{m,m\}) - Pr_{\cI^{0}}(D_t=\{m,m\})\big) \le T. D_{TV}(\cI^0,\cI^{m}),
\end{align} 
where $D_{TV}(\cI^0,\cI^{m})$ denotes the total variation distance between the probability distribution of $\cI^0$ and $\cI^{m}$ with respect to $\cH_T$, i.e. $D_{TV}(\cI^0,\cI^{m}): = \sup_{\cE \in \cH_T}|Pr_{\cI^0}(\cE)-Pr_{\cI^{m}}(\cE)|$, $\cH_T = \sigma\big(\{\P_t(x_t,y_t)\}_{t \in T}\big)$ being the sigma algebra generated by the observed history till time $T$.

Further using Pinksker's inequality we have 
\begin{align}
\label{eq:lb2}
D_{TV}(\cI^0,\cI^{m}) \le \sqrt{\frac{1}{2}D_{KL}(\cI^0,\cI^{m})},
\end{align}
where $D_{KL}(\cI^0,\cI^{m})$ denotes the KL-divergence between the probability distribution induced on the observed history $\cH_T$ by the problem instance $\cI^0$ and $\cI^{m}$. 
Thus averaging over $\cI^{m}$s for all $m \in [\tK]$: 
\begin{align}
\label{eq:lb00}
\nonumber &\frac{\sum_{m \in [\tK]}\E_{\cI^{m}}[N_T(m,m)]}{\tK}  \le \frac{\sum_{m \in [\tK]}\big(\E_{\cI^{0}}[N_T(m,m)]   + T. D_{TV}(\cI^0,\cI^{m})\big)}{\tK}\\
\nonumber & = \frac{\sum_{m \in [\tK]}\E_{\cI^{0}}[N_T(m,m)]}{\tK} + T\sum_{m \in [\tK]} \frac{1}{\tK}\big(  \sqrt{\frac{1}{2}D_{KL}(\cI^0,\cI^{m})}\big)\\
& = \frac{\sum_{m \in [\tK]}\E_{\cI^{0}}[N_T(m,m)]}{\tK} + T\sqrt{\bigg(  \frac{1}{2\tK}\sum_{m \in [\tK]}D_{KL}(\cI^0,\cI^{m})\bigg)}
\end{align}

Now with slight abuse of notation, by denoting $\cI^0_t := Pr_{\cI^0}\big(\P_t(x_t,y_t) \mid \cH_{t-1}\big)$ and $\cI^{m}_t := Pr_{\cI^{m}}\big(\P_t(x_t,y_t) \mid \cH_{t-1}\big)$, we note that: 
\[
D_{KL}(\cI^0_t,\cI^{m}_t) \sim \begin{cases} 
KL\big(\text{Ber}(0.9),\text{Ber}(0.9+\epsilon )\big), ~\text{if }  D_t=\{m,n\} \text{ for any } n \in [K]\sm[\tK]\\
0, ~\text{otherwise}
\end{cases}.
\] 

Further using chain rule of KL-divergence we get:

\begin{align*}
& D_{KL}(\cI^0,\cI^{m}) = \sum_{t = 1}^{T}D_{KL}(\cI^{0}_t, \cI^{m}_t) = \sum_{t = 1}^{T}\sum_{n = \tK+1}^{K}Pr_{\cI^0}(D_t=\{m,n\})D_{KL}\big( \text{Ber}(0.9),\text{Ber}(0.9+\epsilon) \big) \\
& \le D_{KL}\big( \text{Ber}(0.9),\text{Ber}(0.9+\epsilon) \big)\sum_{t = 1}^{T}\sum_{n = \tK+1}^{K}Pr_{\cI^0}(D_t=\{m,n\}) \le 90K\epsilon^2 \sum_{n = \tK+1}^{K}\E_{\cI^0}[N_T(m,n)], 
\end{align*}
where the last inequality follows by noting $D_{KL}\big( \text{Ber}(0.9),\text{Ber}(0.9+\epsilon) \big) \le 90\epsilon^2$ for any $\epsilon \in (0,0.1)$. Further averaging over $\cI^{m}$s for all $m \in [\tK]$: 

\begin{align*}
\nonumber &\frac{\sum_{m \in [\tK]}D_{KL}(\cI^0,\cI^{m}) }{\tK}  \le \frac{90 \epsilon^2\sum_{m \in [\tK]} \sum_{n = \tK+1}^{K}\E_{\cI^0}[N_T(m,n)]}{\tK} \le \frac{90\epsilon^3 T}{\tK} 
\end{align*}

where the last inequality follows due to Observation $2$. 
Now combining above with Eqn. \eqref{eq:lb0} and \ref{eq:lb00} we get:

\begin{align*}
\E & [R_T^s(\cA)] 
\ge  \frac{\epsilon}{2}\bigg(T - \frac{\sum_{m \in [\tK]}\E_{\cI^m}[N_T(m,m)]}{\tK}\bigg)\\
& \ge \frac{\epsilon}{2}\Bigg(T - \bigg(\frac{\sum_{m \in [\tK]}\E_{\cI^{0}}[N_T(m,m)]}{\tK} + T\sqrt{\Big(  \frac{1}{2\tK}\sum_{m \in [\tK]}D_{KL}(\cI^0,\cI^{m})\Big)}\bigg) \Bigg) \\
& \ge \frac{\epsilon}{2}\Bigg(T - \bigg(\frac{T}{\tK} + T\sqrt{\Big(\frac{90\epsilon^3 T}{K}\Big)}\bigg) \Bigg)
\overset{(1)}{\ge} \frac{\epsilon}{2}\Bigg(T - \bigg(\frac{2T}{K} + T\sqrt{\frac{1}{16}}\bigg) \Bigg) \ge \frac{\epsilon}{2}\Big( T - \frac{3T}{4} \Big) = \frac{\epsilon T}{8}
\end{align*}
where $(1)$ holds since by the assumption of Case \textbf{1} we have $T \le \frac{K}{1440 \epsilon^3}$, and the last inequality follows for any $K \ge 4$. This gives a regret lower bound for Case \textbf{1}.

\textbf{Case 2. \big($T > \frac{K}{1440 \epsilon^3}$\big):}
Let us denote by $T_0 = \frac{K}{1440 \epsilon^3}$, and first assume that there exist a $T' > T_0$ such that $R^s_{T'} \le \frac{K}{115200 \epsilon^2}$. However this implies $R^s_{T_0} \le R^s_{T'} \le \frac{K}{115200 \epsilon^2} = \frac{\epsilon T_0}{8}$. But this is a contradiction as per the lower bound of Case \textbf{1}. Thus for any $\epsilon \in (0,0.1)$ and $T > \frac{K}{1440 \epsilon^3}$, any learning algorithm must incur at least an expected regret of $R^s_T \ge \frac{K}{115200 \epsilon^2}$.

The final regret lower bound now follows combining the lower bounds of Case \textbf{1} and \textbf{2}, and from the fact that $R_T \ge {R_T^s}$ (as per Rem. \ref{rem:regeqv}). 

\end{proof}
